\documentclass[twoside]{article}

\usepackage[accepted]{aistats2022}
\usepackage[hidelinks,colorlinks=true,linkcolor=blue,citecolor=blue]{hyperref}
\usepackage[round]{natbib}

\usepackage[ruled,vlined]{algorithm2e}
\usepackage{amsmath}
\usepackage{amsthm}
\usepackage{amssymb}
\usepackage{stmaryrd}
\usepackage{subfiles}
\usepackage{dsfont}
\usepackage{booktabs}
\usepackage{graphicx}
\usepackage{xcolor}

\theoremstyle{plain}

\theoremstyle{proposition}
\newtheorem{proposition}{Proposition}[section]
\theoremstyle{lemma}

\newcommand\XS{\boldsymbol{X}_S}
\newcommand\X{\boldsymbol{X}}
\newcommand\xs{\boldsymbol{x}_S}
\newcommand\XSb{\boldsymbol{X}_{\bar{S}}}

\newcommand\x{\boldsymbol{x}}

\begin{document}

%

%

\twocolumn[

\aistatstitle{Accurate Shapley Values for explaining tree-based models}

\aistatsauthor{ Salim I. Amoukou  \And  Tangi Salaün \And Nicolas J.B. Brunel}

\aistatsaddress{ University Paris Saclay \\ LaMME \\ Stellantis \And Quantmetry \And  University Paris Saclay, ENSIIE \\ LaMME \\ Quantmetry} ]

\begin{abstract}
Shapley Values (SV) are widely used in explainable AI, but their estimation and interpretation can be challenging, leading to inaccurate inferences and explanations. As a starting point, we remind an invariance principle for SV and derive the correct approach for computing the SV of categorical variables that are particularly sensitive to the encoding used. In the case of tree-based models, we introduce two estimators of Shapley Values that exploit the tree structure efficiently and are more accurate than state-of-the-art methods. Simulations and comparisons are performed with state-of-the-art algorithms and show the practical gain of our approach. Finally, we discuss the limitations of Shapley Values as a local explanation. These methods are available as a \href{https://github.com/salimamoukou/acv00}{Python package}.
\end{abstract}

\addtocontents{toc}{\protect\setcounter{tocdepth}{0}}
\section{Introduction}
\label{introduction}
The increasing use of Machine Learning (ML) models in industry, business, sciences, and society has brought the interpretability of these models to the forefront of ML research. As ML models are often considered as black-box models, there is a growing demand from scientists, practitioners, and citizens for tools that can provide insights into important variables in predictions or identify biases for specific individuals or sub-groups. Standard global importance measures like permutation importance measures \citep{breiman2001random} are insufficient for explaining individual or local predictions, and new methodologies are being developed in the active field of Explainable AI (XAI).

In this context, various local explanations have been proposed, focusing on model-agnostic methods that can be applied to the most successful ML models, such as ensemble methods like random forests and gradient boosted trees, as well as deep learning models. Some of the most widely used methods are the Partial Dependence Plot \citep{friedman2001}, Individual Conditional Expectation \citep{goldstein2014peeking}, and local feature attributions such as Local Surrogate (LIME) \citep{ribeiro2016why}. These techniques enable better understanding of the predictions made by a model for individual cases, providing transparency and trust in the ML models' decision-making processes. To achieve the same objective, Shapley Values \citep{shapley1953}, a concept developed primarily in Cooperative Game Theory, has been adapted for XAI to evaluate the "fair" contribution of a variable $X_i = x_i$ to a prediction $f(x_1, \dots, x_p)$\citep{strumbelj2010, NIPS2017_7062}. Shapley Values (SV) are widely used to identify important variables at both local and global levels. As remarked by \cite{lundberg2020local2global,covert2020explaining}, a lot of importance measures aim at analyzing the behavior of a prediction model $f$ based on $p$ features $X_1,\dots,X_p$ by removing variables and considering reduced predictors. Typically, for any group of variables $\XS=\left(X_{i}\right)_{i\in S}$, with any subset $S\subseteq\left\llbracket 1,p\right\rrbracket $ and reference distribution $Q_{S, \boldsymbol{x}}$,  the reduced predictor is defined as:
\begin{align}
  f_{S}(\boldsymbol{x}_{S})\triangleq E_{Q_{S,\boldsymbol{x}}}\left[f\left(\boldsymbol{x}_{S},\boldsymbol{X}_{\bar{S}}\right)\right],
 \label{eq:ReducedPredictor} 
\end{align}
where $Q_{S, \boldsymbol{x}}$ is the conditional distribution $\XSb\vert \XS = \xs$. Other SV can be defined with the marginal probabilities but their interpretation is different \citep{heskes2020causal, janzing2019feature, chen2020true}. There are still active debates on using or not conditional probabilities \citep{frye2020shapley}. This work focuses only on conditional SV, as estimating them poses significant challenges. The SV for explaining the prediction $f(\boldsymbol{x})$ have been introduced in \citep{NIPS2017_7062} and are based on a cooperative game with value function $v(S) \triangleq f_{S}(\boldsymbol{x}_{S})$. For any group of variables $C\subseteq\left\llbracket 1,p\right\rrbracket$ and $k\in\left\llbracket 1,p-\left|C\right|\right\rrbracket $, we denote the set $\mathcal{S}_{k}(C)=\Big\{ S\subseteq\left\llbracket 1,p\right\rrbracket \backslash C : \left|S\right|=k\Big\}$ and we introduce a straightforward generalization of the SV for coalition $C$ as 
\begin{multline}
 \phi_{X_C}(f) \triangleq \frac{1}{p-\left|C\right|+1} \sum_{k=0}^{p-\left|C\right|}\frac{1}{\binom{p-\left|C\right|}{k}} \\
 \sum_{S\in\mathcal{S}_{k}(C)} \big[f_{S\cup C}(\boldsymbol{x}_{S\cup C}) - f_{S}(\boldsymbol{x}_{S})\big].
\label{eq:classic_shapley_game} 
\end{multline}
This definition of the Shapley Value serves as a generalization of the classical SV for a single variable. By considering the singleton $C=\{i\}$ for $i\in\left\llbracket 1,p\right\rrbracket$, we retrieve the standard definition for feature $X_i$. In the following section, we show how this definition arises naturally when measuring the impact of a group of variables $C$, particularly in the case of categorical variables.

We propose solutions for computing and estimating the Shapley Values (SV). We focus solely on tree-based models due to their reduced computational cost and easier statistical handling. We demonstrate that the current state-of-the-art algorithm for tree-based models, TreeSHAP \citep{lundberg2020local2global}, is highly biased when features are dependent. Consequently, we introduce statistically principled estimators to improve the estimation of the SV. Additionally, we address the theoretical computation of SV for categorical variables when using standard encodings, which motivates the use of equation \ref{eq:classic_shapley_game}. Specifically, we show that the true SV of a categorical variable is different from the sum of SVs of encoded variables, as generally used. Moreover, using the sum of encoded variables as the SV of a categorical variable provides incorrect estimates of all SVs in the model and leads to spurious interpretations. This is currently the only way to handle categorical variables with TreeSHAP, and we therefore highlight the correct method for computing the SV of encoded variables and implement it using our estimators. Our contributions, which reduce bias in SV estimation, are implemented in a \href{https://github.com/salimamoukou/acv00}{Python package}.

The paper is structured as follows. In the next section, we derive invariance principles for SV under reparametrization or encoding, which is particularly useful for dealing with categorical variables. In section 3, we introduce two estimators of reduced predictors and SV. In section 4, we highlight the improvement over dependent TreeSHAP. Finally, we discuss the reliability of SV in providing local explanations.

\section{Coalition and Invariance for Shapley Values \label{sec:Invariance}}

In this section, we present a unifying property of invariance for the Shapley Values of continuous and categorical variables. The property states that the explanation provided by a variable should not depend on the way it is encoded in a model. This invariance property provides a natural way of computing the SV of categorical variables based on the notion of coalition and the general definition given in equation \ref{eq:classic_shapley_game}. This is especially useful in our case, as we are also interested in the discretization of continuous variables to facilitate the estimation of Shapley Values and enhance their stability, which we will discuss in Section  \ref{sec:TreeBased}.

\subsection{Invariance under reparametrization for continuous variables}
In equation \ref{eq:classic_shapley_game}, there is no restriction on the dimension of $X_i$. We assume that the $p$ variables are vector-valued, i.e., $X_i \in \mathbb{R}^{p_i}$ where $p_i\geq 1$. We further assume that each variable $X_i$ is transformed with a diffeomorphism $\varphi_i: \mathbb{R}^{p_i} \longrightarrow \mathbb{R}^{p_i}$. We introduce the transformed variables $U_i \triangleq \varphi_i(X_i)$ and the reparametrized model $\Tilde{f}$ defined by $\Tilde{f}(U_1,\dots,U_p) = f( X_1,\dots,X_p)$, i.e., $\Tilde{f}(U_1,\dots,U_p) = f \circ \boldsymbol{\varphi}^{(-1)}(\boldsymbol{\boldsymbol{U}})$ where $\varphi = (\varphi_1,\dots,\varphi_p)$. Generally, we cannot relate the predictor learned from the real dataset $\{(\X_i, Y_i)\}_{i=1}^{n}$ to the predictor learned from the transformed dataset $\{(\boldsymbol{U}_i,Y_i)\}_{i=1}^n$ where $Y$ is the label to predict. Estimation procedures are not invariant with respect to reparametrization, which means we obtain different predictors after "diffeomorphic feature engineering": $f_{\boldsymbol{U}} \neq f_{\X} \circ \varphi$. Consequently, we focus only on the impact of reparametrization on explanations, and we show below that the Shapley Values are invariant under reparametrization.
\begin{proposition}\label{prop:reparam}
Let $f$ and $\Tilde{f} = f \circ \boldsymbol{\varphi}^{(-1)} $ its reparametrization, then we have $\forall i \in \left\llbracket 1,p\right\rrbracket$, and $\boldsymbol{U} = \boldsymbol{\varphi}(\X)$: 
\begin{align}
     \phi_{X_i}(f) = \phi_{U_i}(\tilde{f}). \label{eq:equivariance} 
\end{align}
\end{proposition}
The proof can be found in the Supplementary Material. This identity indicates that the information provided by each feature $X_i$ in the explanation is independent of any encoding, as mentioned by \cite{owen2017shapley, sage}. The SV primarily depends on the dependence structure of the features. Therefore, the Shapley Values of a feature $X_i$ remains the same after diffeomorphic transformation $\varphi$, we have $\phi_{X_i}(f) = \phi_{U_i}(\tilde{f})$. Suppose a variable $X_i$ is separated into $C$ correlated variables $\Tilde{X}_C = (X_i)_{i \in C}$, for instance, by discretizing a feature. As $X_i$ and $\Tilde{X}_C$ carry the same information, we may ask whether the SV of the group of features $\Tilde{X}_C$ is equal to the SV of $X_i$. In the next section, we provide an affirmative answer to this question.

\subsection{Invariance for encoded categorical variable}
In the remainder of the paper, we use $X$ to denote continuous predictive variables, $Z$ to denote categorical variables and $Y$ to denote the output of interest. While there are numerous encodings for a categorical variable $Z$ with modalities $\{1,\dots,K\}$, we focus on two popular methods: One-Hot-Encoding (OHE) and Dummy Encoding (DE). These methods introduce indicator variables $Z_k$ such that $Z_{k}=1$ if $Z=k$, $0$ otherwise.  
In contrast to the continuous case,  introducing indicator variables changes the number of "players" in the game defined for computing the Shapley Value. This change has significant consequences on the computation of the SV for all variables in the model. Hence, the widely adopted practice of summing the SV of indicator variables $Z_k$ to compute the SV of $Z$ is generally not justified and false. To benefit from a similar invariance result as Proposition \ref{prop:reparam}, we need to deal with the coalition of indicators and use the general expression of SV introduced in equation  \ref{eq:classic_shapley_game}.
For the sake of simplicity, we assume that the model $f$ uses only the two variables $(X, Z)$, where $X\in\mathbb{R}$ and $Z \in \{1,\dots,K\}$ is a categorical variable. The efficiency property of SV gives the decomposition
\begin{equation} 
f(x,z)-E_{P}\left[f\left(X,Z\right)\right]=\phi_{X}(f)+\phi_{Z}(f)\label{eq:Shapley_decomposition_2var}
\end{equation}

where $P$ denotes the law of $(X, Z)$. To establish the link between the SV of the indicator variables $Z_{k}$ and the SV of the variable $Z$, we introduce additional notations. We focus on the Dummy Encoding (DE) $\varphi:z\mapsto(z_{1},\dots,z_{K-1})$ without loss of generality. The variables $\left(X,Z_{1:K-1}\right)$ are defined on $\mathbb{R}\times \left\{0,1\right\}^{K-1}$, and their distribution $\tilde{P}$ is the image probability of $P$ induced by the transformation $\varphi$. The initial predictor  $f:\mathbb{R}\times\left\{ 1,\dots,K\right\} \longrightarrow\mathbb{R}$
is reparametrized as a function  $\tilde{f}:\mathbb{R}\times\left\{ 0,1\right\}^{K-1} \longrightarrow\mathbb{R}$
such that $f(X,Z)\triangleq\tilde{f}(X,Z_{1},\dots,Z_{K-1})$. The function $\tilde{f}$ is not completely defined for all $(z_{1}\dots,z_{K-1})\in \left\{ 0,1\right\}^{K-1}$
and is only defined $\tilde{P}$-almost everywhere because of the
deterministic dependence $\sum_{k=1}^{K-1}Z_{k}\leq1$. Consequently, we need to extend $\tilde{f}$ to the whole space $\mathcal{X}\times \left\{ 0,1\right\}^{K-1}$ by setting $\tilde{f}(x,z_{1},\dots z_{K-1})=0$ as soon as $\sum_{k=1}^{K-1}z_{k}>1$. For the predictor $\tilde{f}(X,Z_{1},\dots,Z_{K-1})$, we can compute the SV of $X,Z_{1},\dots,Z_{K-1}$ and obtain the following decomposition thank to the efficiency property of SV,
\begin{multline}  \tilde{f}(x, z_{1}, \dots, z_{K-1})-E_{\tilde{P}}\left[\tilde{f}\left(X, , Z_{1}, \dots, Z_{K-1}\right)\right] \\
= \phi_{X}(\tilde{f}) + \sum_{k=1}^{K-1}\phi_{Z_k}(\tilde{f})  \label{eq:ShapleyValue_XDummy_Indiv}
\end{multline}
where $\phi_{Z_k}(\tilde{f})$ are the SV of
the variable $Z_{k}$ computed with distribution $\tilde{P}$ and model $\tilde{f}$. Note that
\begin{align*}
    & f(x,z)-E_{P}\left[f\left(X,Z\right)\right]\\
    & = \tilde{f}(x, z_{1}, \dots, z_{K-1})-E_{\tilde{P}}\left[\tilde{f}\left(X, , Z_{1}, \dots, Z_{K-1}\right)\right].
\end{align*}
As a result, we have
\begin{align} 
\phi_{X}(f)+\phi_{Z}(f) =\phi_{X}(\tilde{f}) + \sum_{k=1}^{K-1}\phi_{Z_k}(\tilde{f})\label{eq:EqualityShapleyValues}
\end{align}
In general, we have $\phi_Z(f)\neq \sum_{k=1}^{K-1}\phi_{Z_k}(\tilde{f})$, because the SV depends on the number of variables and they are not calculated using the same quantities.  We show in the next proposition that $\phi_Z(f) = \phi_{Z_{1:K-1}}(\tilde{f})$ where $\phi_{Z_{1:K-1}}(\tilde{f})$ is computed with equation \ref{eq:classic_shapley_game} and corresponds to the Shapley Values of the coalition of the variables $(Z_1, ..., Z_{K-1})$.

\begin{proposition}\label{prop:SV_CoalQualitative}
 Given a predictor  $f:\mathbb{R}\times\left\{ 1,\dots,K\right\} \longrightarrow\mathbb{R}$ and its reparametrization $\tilde{f}$ using Dummy Encoding  $\tilde{f}:\mathbb{R}\times\left\{ 0,1\right\}^{K-1} \longrightarrow\mathbb{R}$ such that $f(X,Z)\triangleq\tilde{f}(X,Z_{1},\dots,Z_{K-1})$, we have
 
 \begin{equation}
\left\{ \begin{array}{lll}
 \phi_{Z_{1:K-1}}(\tilde{f}) & = & \phi_{Z}(f)\\
 \phi_{X}(\tilde{f}; Z_{1:K-1}) & = & \phi_{X}(f),
\end{array}\right.\label{eq:Shapley_CoalitionDummies}
\end{equation}
\end{proposition}
where $\phi_{X}(\tilde{f}; Z_{1:K-1})$ is the SV of $X$ when the variables $(Z_1, ..., Z_{K-1})$ are considered as a single variable. We refer to Supplementary Material for detailed derivations. In general, for cooperative games, the SV of a coalition
$\phi_{Z_C}(\tilde{f})$ with $C \subseteq \{1, \dots, K-1\}$ is different from the sum of individual
SV $\sum_{k\in C}\phi_{Z_k}(\tilde{f})$. We
 note that we can compute two different SV for $X$ when we use the reparametrized predictor $\tilde{f}$: $\phi_{X}(\tilde{f})$ and $\phi_{X}(\tilde{f}; Z_{1:k})$. These two SV are different in general as they involve different numbers of variables and different conditional expectations. Proposition \ref{prop:SV_CoalQualitative} shows that we should prefer $\phi_{X}(\tilde{f}; Z_{1:k})$  as it is equal to the SV of $X$ in the original model $\phi_X(f)$.

\subsection{Coalition or Sum: numerical comparisons}
We give numerical examples illustrating the differences between coalition or sum on the corresponding explanations. We consider a linear predictor $f$, with 1 categorical and 3 continuous variables $(X_1, X_2, X_3)$, defined as $f(X, Z) = B_Z X$ with $X|Z=z \sim \mathcal{N}(\mu_z, \Sigma_z)$ and $\mathbb{P}(Z=z)= \pi_z$, $Z \in \{a, b, c\}$. The values of the parameters used in our experiments are found in the Supplementary Material. In figure \ref{fig:toy_model_examples_encoding}, we remark that the SV change dramatically w.r.t the encoding. The sign changes given the encoding (DE or OHE) and is often different from the sign of the true SV of $Y$ without encoding. We can also note important differences in the SV of the quantitative variable $X$. 

To quantify the global difference of the different methods, we compute the relative absolute error (R-AE) of the SV defined as:
\begin{equation} \label{eq:R-AE}
    \text{R-AE}(f, \tilde{f}) = \sum_{i=1}^{p} \frac{|\phi_{X_i}(f) - \phi_{X_i}(\tilde{f})|}{|\phi_{X_i}(f)|}
\end{equation}

\begin{figure}[ht]
    \centering
    \includegraphics[width=0.7\linewidth]{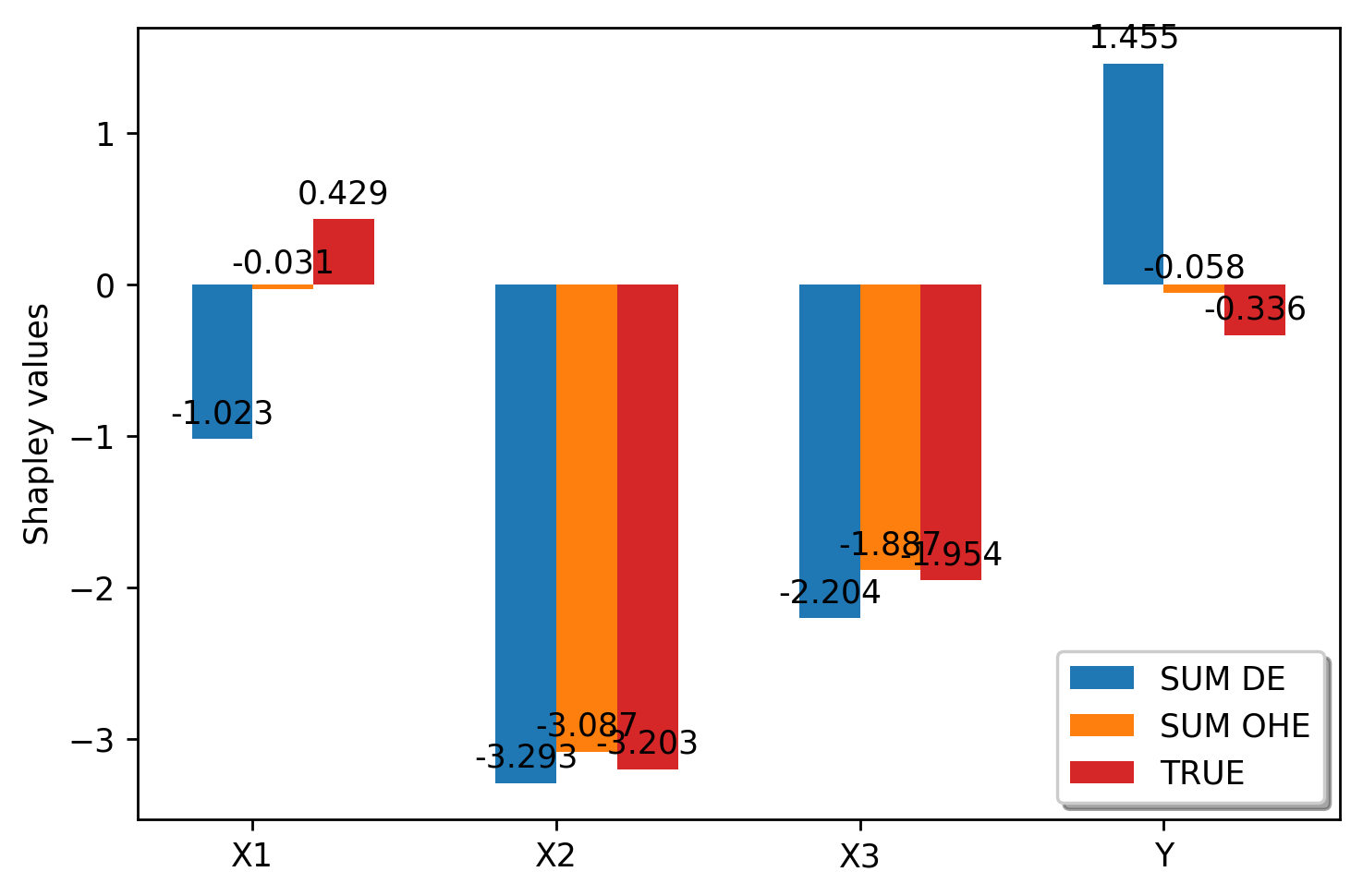}
    \caption{SV with or without encoding (OHE - DE) for observation $(X_1, X_2, X_3) = [0.35, -1.61, -0.11], Z = a$}
    \label{fig:toy_model_examples_encoding}
\end{figure}

We compute the SV of 1000 new observations. We observe in figure \ref{fig:toy_model_examples_error} that the differences can be huge for almost all samples (DE is much worse than OHE in this example). Thus, we highly recommend using the coalition as it is consistent with the true SV contrary to the sum.  More examples on real datasets can be found in the Supplementary Material.
\begin{figure}[ht]
    \centering
    \includegraphics[width=0.7\linewidth]{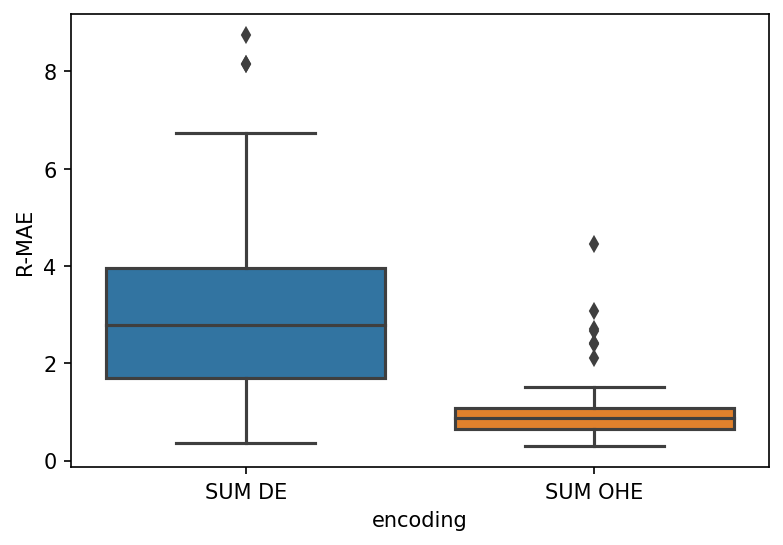}
    \caption{R-AE distribution between the SV without encoding and the corresponding encodings}
    \label{fig:toy_model_examples_error}
\end{figure}

\section{Shapley Values for tree-based models\label{sec:TreeBased}}
\label{shap_values_estimations}
The computation of Shapley Values (SV) faces two main challenges: the combinatorial explosion with $2^p$ coalitions to consider and the estimation of conditional expectations $f_S(\xs)=E\left[f(\boldsymbol{X})\vert\boldsymbol{X}_{S} = \xs\right], S\subseteq\left\llbracket 1,p\right\rrbracket$. Current approaches rely on several approximations and sampling procedures that assume independence \citep{NIPS2017_7062, improvingkernelCovert}. More recently, some methods propose to model the joint distribution of features with a gaussian distribution or vine copula to draw samples from the conditional distributions \citep{aas2020explaining, aas2021explaining}. Other methods, such as \citep{williamson2020efficient}, train one model for each selected subset $S$ of variables, which is accurate but computationally costly.  However, their final objective differs from ours since we are interested in local estimates and exact computations (i.e., no sampling of the subsets). To achieve this, we focus on tree-based models, as exploited by \cite{lundberg2020local2global} for deriving an algorithm (TreeSHAP) for exact computation of SV, where we can compute all the terms (no sampling of the subsets $S\subseteq\left\llbracket 1,p\right\rrbracket$) and the estimation of the conditional expectations is simplified.
After briefly presenting the limitations of TreeSHAP, we introduce two new estimators that use the tree structure.  For simplicity, we consider a single tree and not an ensemble of trees (Random Forests, Gradient Tree Boosting, etc.), as extending our estimators to these more complex models is straightforward through linearity.

\subsection{Algorithms for computing Conditional Expectations and the Tree SHAP algorithm}
We consider a tree-based model $f$ defined on $\mathbb{R}^p$ (categorical variables are one-hot encoded). We have 
$f(\boldsymbol{x}) = \sum_{m=1}^{M} f_m \mathds{1}_{L_m}(\boldsymbol{x})$ where $L_m$ represents a leaf. The leaves form a partition of the input space, and each leaf can be written as $L_{m}=\prod_{i=1}^{p}\left[a_{i}^{m},b_{i}^{m}\right]$
with $-\infty\leq a_{i}^{m}<b_{i}^{m}\leq+\infty$. 
Alternatively, we  write the leaf with the decision path perspective: a leaf $L_m$ is defined by a sequence of decision based on $d_m$ variables $X_{N_k^m},\, k=1,\dots,d_m, N_k^m \in \{1,\dots, p\}$. For each node $k$ in the path of the leaf $L_m$, $X_{N_k^m}$ is the variable used to split, and the region $I_{k}^m$ defined by the split value $t_k^m$ is either $]-\infty, t_k^m]$ or $[t_k^m,+\infty[$. The leaf can be rewritten as
\begin{align}
     L_m = \left\{ \boldsymbol{x}\in \mathbb{R}^p: x_{N_1^m} \in I_{1}^m, \dots, x_{N_{d_m}^m}\in I_{d_m}^m\right\}.
    \label{eq:lm}
\end{align}
The crucial point is to identify the set of leaves compatible with the condition $\XS=\xs$. We can partition the leaf according to a coalition $S$ as $L_{m}=L_{m}^{S}\times L_{m}^{\bar{S}}$ with $L_{m}^{S}=\prod_{i\in S}\left[a_{i}^{m},b_{i}^{m}\right]$
and $L_{m}^{\bar{S}}=\prod_{i\in\bar{S}}\left[a_{i}^{m},b_{i}^{m}\right]$. Thus, for each condition $\boldsymbol{X}_{S}=\boldsymbol{x}_{S}$ the set of compatible leaves of $\boldsymbol{x}=(\boldsymbol{x}_{S},\boldsymbol{x}_{\bar{S}})$ is
\begin{align*}
     C\left(S,\boldsymbol{x}\right)& =\left\{ m\in\left[1\dots M\right]: \boldsymbol{x}_{S}\in L_{m}^{S}\right\}\\
     & = \left\{ m\in\left[1\dots M\right] : x_{N_i^m} \in I_{i}^m,\, N_i^m \in S \right\}
\end{align*}
and the reduced predictor $f_S(\xs)$ has the simple expression
\begin{equation*}
     f_S(\xs) = \sum_{m\in C(S,\boldsymbol{x})}  f_m P_X(L_m \vert \XS=\xs) 
\end{equation*}
where the probability is computed under the law of the features $P_X$.  If we have a model for $P_X$, we can derive the conditional law and directly evaluate the conditional probabilities. For instance, when $\boldsymbol{X} \sim \mathcal{N}(\mu, \Sigma)$, we can compute exactly the conditional probabilities $P_X(L_m | \boldsymbol{X}_S=\boldsymbol{x}_S) = P_X \left( \prod_{k=1}^{d_m} I_{k}^m \vert \XS=\xs \right)$. In general, deriving conditional probabilities can be challenging, but assumptions about the factorization of the distribution can accelerate the computation. In \citep{lundberg2020local2global}, the authors introduce a recursive algorithm (TreeSHAP with path-dependent feature perturbation, Algorithm 1) that assumes that the probabilities for every compatible leaf $L_m$ can be factored with the decision tree, which simplifies the computation as
{ \small 
\begin{multline} 
   \small  P_X^{SHAP} \left( \prod_{k=1}^{d_m} I_{k}^m \; \Big| \XS=\xs\right) = \delta_S(N_1^m) \times \\
    \small  \prod_{i=2: N_i^m \notin S}^{d_m }  P\left(X_{N_{i}} \in I_{i}^m \; \Big| \prod_{k=2 : N_k^m \notin S}^{i} X_{N_{k-1}} \in I_{k-1}^m\right) \label{eq:ProbaSHAP}
\end{multline}}with $\delta_S(N_1)=P(X_{N_{1}^m} \in I_1^m)$ if $N_1\notin S$, and 1 otherwise. The underlying assumption in (\ref{eq:ProbaSHAP}) is that we have a Markov property defined by the path in the tree, see the algorithm description in the Supplementary Material. However, as we will demonstrate in our simulations, this assumption is too strong and leads to a high estimation bias.  We denote $\widehat{f}_S^{SHAP}$ and $\phi_{X_i}(\widehat{f}_S^{SHAP})$ as the estimators of the reduced predictor and Shapley Values using the estimator (\ref{eq:ProbaSHAP}). Therefore, we propose two estimators that do not rely on the factorization of $P_X$.

\subsection{Statistical Estimation of Conditional Expectations \label{subsec:Closed-form-Expectations} }
\textbf{Discrete case}. To address the statistical problem of estimating the probabilities $P_X(L_m | \boldsymbol{X}_S=\boldsymbol{x}_S)$ from a given dataset $\mathcal{D} =\{\X_i\}_{i=1}^{n}$, where $\X_i \sim P_X$, without assuming any density or prior knowledge about $P_X$ as done in \cite{aas2020explaining, aas2021explaining}, we first consider the case where all variables are categorical. This allow us to estimate $P_{X}\left(L_{m}\left|\XS=\xs\right.\right)$ directly. A straightforward estimation is based on $N(\boldsymbol{x}_{S})$, which is the number of observations in $\mathcal{D}$ such that $\boldsymbol{X}_{S}=\boldsymbol{x}_{S}$, and $N(L_{m},\boldsymbol{x}_{S})$, which is the number of observations in leaf $L_{m}$ of $\mathcal{D}$ that satisfy the condition $\boldsymbol{X}_{S}=\boldsymbol{x}_{S}$. An accurate estimation of the conditional probability $P_X(L_m | \boldsymbol{X}_S=\boldsymbol{x}_S)$ can be obtained by computing the ratio of these two terms as 
\begin{align}
\widehat{P}_{X}^{(D)}\left(L_{m}\left|\boldsymbol{X}_{S} = \boldsymbol{x}_{S}\right.\right) = \frac{N(L_{m},\boldsymbol{x}_{S})}{N(\boldsymbol{x}_{S})}.    
\label{eq:cat_pluging}
\end{align}
Estimating the conditional probabilities $P_X(L_m | \boldsymbol{X}_S=\boldsymbol{x}_S)$ becomes more challenging when the variables $\boldsymbol{X}_{S}$ are continuous. A common approach is to use kernel smoothing estimators \citep{nadaraya1964estimating}. However, this method has several drawbacks, such as a low convergence rate in high dimensions and the need to derive and select appropriate bandwidths, which can add complexity and instability to the estimation procedure. To address these issues, we propose a simple approach based on quantile-discretization of the continuous variables. This technique is commonly used for easing model explainability, particularly in tree-based models, as shown in \citep{benard2021interpretable}. Binning observations can also help to stabilize the reduced predictors and Shapley Values, thus improving the robustness of the explanation \citep{alvarez2018robustness}.\\
In our experiments, we use a simple approach to discretize continuous variables into $q=10$ quantiles, where each feature $X_i$ is encoded with indicator variables $X_i^{(r)}, r\in \{1,\dots,q\}$. Let $\widehat{q}_i^{(r)}$ denote the empirical $(\frac{r}{q})$-th quantile of feature $X_i$ using dataset $\mathcal{D}$, and let $\widehat{q}_i^{(0)} = -\infty$ and $\widehat{q}_i^{(q)} = +\infty$. We define $X_i^{(r)} = 1$ if $X_i$ falls into the interval $[\widehat{q}_i^{(r-1)}, \widehat{q}_i^{(r)})$. To compute the Shapley Values of a given feature $X_i$, we use the coalition of its indicator variables $\left(X_i^{(1)},\dots,X_i^{(q)} \right)$ as defined in Proposition \ref{prop:SV_CoalQualitative}. Then, we define the Discrete reduced predictor denoted by $\widehat{f}_S^{D} (\boldsymbol{x}S)$ as
\begin{equation} \widehat{f}_S^{D} (\boldsymbol{x}_S)= \sum_{m\in C(S,\boldsymbol{x})} f_m \widehat{P}_{X}^{(D)}\left( L_{m}|\XS = \xs \right) \label{eq:CategoryPredictor}
\end{equation}
and the corresponding estimator of the SV is denoted $\phi_{X_i}(\widehat{f}^{D})$. Although the discretization of continuous variables leads to some loss of information, it is often negligible in terms of performance when using tree-based models, as shown in the Supplementary Material. With only $q=10$ quantiles, the input space is divided into a fine grid of $p^{10}$ cells,  which provides a rich representation of the data. However, this is also a limitation, as the number of cells increases rapidly with $p$ and the number of categories per variable, leading to high variance for cells with low frequencies. Therefore, we propose an alternative estimator that leverages the leaf information provided by the decision tree, allowing for faster SV computation.

\textbf{Continuous and mixed-case}. Instead of discretizing the variables, we use the leaves of the estimated trees. Essentially, we replace the conditions $\left\{ \boldsymbol{X}_{S}=\boldsymbol{x}_{S}\right\} $ by $\left\{ \boldsymbol{X}_{S}\in L_{m}^{S}\right\} $.
This change introduces a bias, but it aims at improving the variance during estimation. 
We introduce the Leaf-based estimator as
\begin{equation}  
\small \widehat{f}_S^{(Leaf)} (\xs)= \frac{1}{Z(S,\boldsymbol{x})} \sum_{m\in C(S,\boldsymbol{x})} f_m \widehat{P}_{X}^{(Leaf)}\left(L_{m}\left|\XS \in L_m^S\right.\right) \label{eq:LeafPredictor}
\end{equation} 
where $\widehat{P}_{X}^{(Leaf)}\left(L_{m}\left|\XS \in L_m^S\right.\right)$ is an estimate of the conditional probability, and $Z(S,\boldsymbol{x})$ is a normalizing constant.
The definition of every probability estimate is  
\[\widehat{P}_{X}^{(Leaf)}\left(L_{m}\left|\XS \in L_m^S\right.\right) = \frac{N(L_{m})}{N(L_{m}^{S})} \] 
where $N(L_{m})$ is the number of observations of $\mathcal{D}$ in the leaf $L_{m}$, and $N(L_{m}^{S})$ is the number of observations of $\mathcal{D}$  satisfying the conditions
$\boldsymbol{x}_{S}\in L_{m}^{S}$. Another interpretation of this estimator is that it projects the partition of the tree along the direction defined by the variables $\X_S$. This results in a projected tree that only considers the variables $\XS$, which is then used to estimate the conditional probability $E[f(\X) \vert \XS = \xs]$. It is important to note that the probability estimates do not necessarily sum up to one as we are not conditioning on the same event, i.e.,
\begin{align*}
     \sum_{m\in C(S,\boldsymbol{x})} \widehat{P}_{X}^{(Leaf)}\left(L_{m}\left|\XS \in L_m^S\right.\right) \neq 1.
\end{align*}
Therefore, we introduce a normalizing constant to ensure the probabilities are correctly normalized.
This normalizing constant is defined as 
\[ Z(S, \boldsymbol{x}) = \sum_{m\in C\left(S,\boldsymbol{x}\right)} \frac{N(L_{m})}{N(L_{m}^{S})}.\]
The Leaf-based reduced predictor (\ref{eq:LeafPredictor}) can be computed for continuous and categorical variables, and hence we can compare it with $\widehat{f}_S^{(D)}$ in order to evaluate its bias. In both cases, the main challenge is the computation of $C(S,\boldsymbol{x})$, for every coalition $S$. We show in the next section how the computational complexity of the SV  $\phi_{X_i}(\widehat{f}^{(Leaf)})$ is drastically reduced. Indeed, when we consider the leaf $L_m$, we only have to compute the SV for $d_m$ variables, which corresponds to the variables used for splitting in the leaf $L_m$, and not for all the $p$ variables.  

\textbf{Bias analysis}. Before employing the two proposed estimators to calculate the Shapley Values, we first analyze the bias of these estimators. When the variables are discrete, it is obvious that the discrete estimator $\widehat{f}_S^{(D)}$ is consistent. However, in the case of the leaf estimator $\widehat{f}_S^{(Leaf)}$, we analyze its bias with respect to the true reduced predictor $f_S(\xs)$ below. 
{\small \begin{equation*}
    \begin{split}
        & \small \widehat{f}_S^{(Leaf)}(\xs) - f_S(\xs) \\
    & =  \sum_{m\in C(S,\boldsymbol{x})} f_m \Big[ \widehat{P}_{X}^{(Leaf)}(L_{m} \vert \XS \in L_m^S) - P_X(L_m \vert \XS=\xs) \Big]\\
    & \small = \sum_{m\in C(S,\boldsymbol{x})} f_m \Big[ \textcolor{teal}{\small \widehat{P}_{X}^{(Leaf)}(L_{m}\vert\XS \in L_m^S)  - P_{X}(L_{m}\vert\XS \in L_m^S)} \Big]\\
    & \small \ + \sum_{m\in C(S,\boldsymbol{x})} f_m \Big[\textcolor{red}{P_{X}\left(L_{m}\left|\XS \in L_m^S\right.\right)  - P_{X}\left(L_{m}\left|\XS = \xs\right.\right)}\Big]
    \end{split}
\end{equation*}}The control of the blue term is well established, and its rate of convergence is known. Recently, \cite{grunewalder2018plug} (Proposition 3.2) shows that if $\widehat{P}_{X}^{(Leaf)}\left(L_{m}\left|\XS \in L_m^S\right.\right) \geq n^{-\alpha}$, with $\alpha \in [0, 1/2)$, then $\big|\widehat{P}_{X}^{(Leaf)}\left(L_{m}\left|\XS \in L_m^S\right.\right)  - P_{X}\left(L_{m}\left|\XS \in L_m^S\right.\right)\big| = \mathcal{O}_P(n^{\alpha - 1/2})$.

The second term depends on the quality of the partition obtained from the tree. The intuition behind the effectiveness of tree-based models is that they group observations with similar conditional laws in each cell. Indeed, one of the assumptions to prove the consistency of tree-based models is that the variation of the conditional law is zero in each leaf, i.e., for all $\x \in L_m$ and $r \in \mathbb{R}$, we have $\sup_{\boldsymbol{z} \in L_m} |F(r | \boldsymbol{z}) - F(r|\boldsymbol{x})| \overset{a.s}{\to} 0$ \citep{scornet2015consistency, meinshausen2006quantile, elie2020random}, or alternatively, show that the diameter of the leaves tends to 0 \citep{gyorfi2002distribution}. The latter ensures that the probability $P_X(L_m \vert \XS = \xs)$ varies slightly as we move within a given cell if $\X$ admits a continuous density. Therefore, if the diameter of the leaves tends to 0 as generally assumed for partition-based estimator \citep{gyorfi2002distribution}, the leaf estimator is consistent.

\subsection{Fast Algorithm for Shapley Values with the Leaf estimator}
Here, we focus on the computational efficiency offered by the Leaf estimator. It is well-known that the computation of the Shapley Values has exponential complexity, as we need to compute $2^p$ different coalitions for each observation. However, with the Leaf estimator $\hat{f}_S^{(Leaf)}$, we can reduce the complexity to being exponential in the depth of the tree in the worst case, instead of being exponential in the total number of variables $p$. This is very interesting, as the depth of the tree is rarely above 10 in practice, while $p$ can be very large, spanning different orders of magnitude. The idea is to split the original game into the sum of smaller games, as described by the following proposition.
\begin{proposition} \label{prop:sum_gam}
Consider a tree-based model $f(\boldsymbol{x})=\sum_{m=1}^{M}f_{m}\mathds{1}_{L_{m}}(\boldsymbol{x})$, and let $S_m$ be the set of variables used along the path to leaf $L_m$. For any observation $\x$ and variable $X_i=x_i$, we can decompose its Shapley value $\phi_{X_i}(f^{(Leaf)})$ into the sum of $\# C(S,\boldsymbol{x})$ cooperative games defined on each leaf $m \in C(S,\boldsymbol{x})$ as follows
 \begin{equation} \label{eq:sumGames}
 \phi_{X_i}\left(f^{(Leaf)}\right) = \sum_{m \in C(S,\boldsymbol{x})} \phi_{X_i}^m\left(f^{(Leaf)}\right)    
 \end{equation}
 where $\phi_{X_i}^m\left(f^{(Leaf)}\right)$ is a reweighted version of the Shapley Value of the cooperative game with players $S_m$ and value function $v(f^{(Leaf)},S) =P_X(L_m | \X_S \in L_m^S)$. 
\end{proposition}

Therefore, to compute the Shapley Values, we propose calculating them leaf by leaf using equation \ref{eq:sumGames}. In this approach, the computation of the Shapley Values for the $p$ variables is performed by summing over $\# C(S,\boldsymbol{x})$ cooperative games (leaves), each having a number of variables $\#S_m, m \in C(S,\boldsymbol{x})$ lower than or equal to $D$, which is the maximum depth of the tree. As a result, the computational complexity is $\mathcal{O}(p \times \# C(S,\boldsymbol{x}) \times 2^{D})$ in the worst cases. To improve the computational efficiency, we introduce the \emph{Multi-Games algorithm} which leverages on Proposition \ref{prop:sum_gam}. This algorithm is linear in the number of observations and can handle a large number of variables. However, its complexity is still higher than that of TreeSHAP \citep{lundberg2020local2global}, which is polynomial with $\mathcal{O}(M \times \texttt{D}^2)$. The algorithm is described below, and we use the notations $N(L_m^{\emptyset}) = \sum_{m=1}^{M} N(L_m)$ and $\mathds{1}_{L_m^{\emptyset}}(\boldsymbol{x}_{\emptyset}) = 1$.

\begin{algorithm}[ht] \label{alg:fast_sv}
\SetAlgoLined
 \textbf{Inputs:}  $\boldsymbol{x}, f(\boldsymbol{x}) = \sum_{m=0}^M f_m \mathds{1}_{L_m}(\boldsymbol{x})$\;
 $p = length(\boldsymbol{x})$\;
 $\phi = \text{zeros}(p)$\;
 \For{$m \in C(S, \x)$}{
    \For{$i$  in $[p]$}{
        \If{$i$ not in $S_m$}{
             continue \tcc*[r]{skip to next variable}
        }
        \For{$S \subseteq S_m$}{
        $\phi[i] = \left(\binom{p - 1}{|S|}^{-1} + \sum_{k=1}^{p -|S_m|} \binom{p - |S_m|}{k} \binom{p - 1}{k+ |S|}^{-1} \right) \times \;
        \left( \mathds{1}_{L_m^{S \cup i}}(\boldsymbol{x}_{S\cup i})\frac{N(L_m)}{N(L_m^{S\cup i})} - \mathds{1}_{L_m^{S}}(\boldsymbol{x}_{S})\frac{N(L_m)}{N(L_m^{S})}\right)$
        }
    }
}
 \caption{\emph{Multi-Games Algorithm}}
 \Return $\phi$
\end{algorithm}

\textbf{Remark}: This algorithm is highly parallelizable, as it can be vectorized to compute the Shapley values for multiple observations simultaneously.

\section{Comparison of the estimators}
 To compare the different estimators, we need a model where conditional expectations can be calculated exactly. If $X \sim \mathcal{N}(\mu, \Sigma)$ then $X_{\bar{S}}|X_S$ is also multivariate gaussian with explicit mean vector and covariance matrix; see Supplementary Material. Note that we do not include any comparisons with KernelSHAP as our main goal is to improve upon TreeSHAP which is the state-of-the-art for tree-based models. In addition, most implementation of KernelSHAP is based on the marginal distribution as its aims to be model-agnostic.

\textbf{Experiment 1.} In the first experiment, we consider a dataset $\mathcal{D}_n =\left\{\boldsymbol{X}_{i}, Y_{i}\right\}_{i=1}^n$ with $n=10^4$ generated by a linear regression model with $\boldsymbol{X} \in \mathbb{R}^{p}$ following a multivariate Gaussian distribution with mean vector $0$ and covariance matrix $\Sigma = \rho J_p + (\rho - 1) I_p$, where $p=5$ and $\rho=0.7$. The response variable is $Y = B^t \boldsymbol{X}$, where $B$ is a vector of coefficients. We trained a RandomForest $f$ on $\mathcal{D}_n$ and obtained an MSE of $4.28$. The parameters used for the training can be found in the Supplementary Material. Since the law of $\boldsymbol{X}$ is known, we can compute the exact Shapley Values (SV) of $f$ using a Monte Carlo estimator (MC).

We aim to compare the true Shapley Value $\phi_{X_i}(f)$ with the Shapley Value estimated by the different methods $\phi_{X_i}(\hat{f}^{(method)})$, where the $method$ can be SHAP, Leaf, or D. To quantify the differences between these estimators, we consider two evaluation metrics. Firstly, we compute the Relative Absolute Error (R-AE) as defined in equation (\ref{eq:R-AE}). Secondly, we measure the True Positive Rate (TPR) to evaluate whether the ranking of the top $k=3$ highest and lowest Shapley Values is preserved across different estimators.

In figure \ref{fig:shap_dist_error}, we compute the SV $\phi_{X_i}(\hat{f}^{SHAP})$, $\phi_{X_i}(\hat{f}^{Leaf})$ on a new dataset of size 1000 generated by the synthetic model. We observe that the estimator $\hat{f}^{Leaf}$ is more accurate than TreeSHAP $\hat{f}^{SHAP}$ by a large margin. TreeSHAP has an average R-AE$=3.31$ and TPR$=86\% (\pm 17\%)$ while Leaf estimator gets R-AE$=0.90$ and TPR$=94\% (\pm12\%)$.

\begin{figure}[ht]
    \centering
    \includegraphics[width=0.7\linewidth]{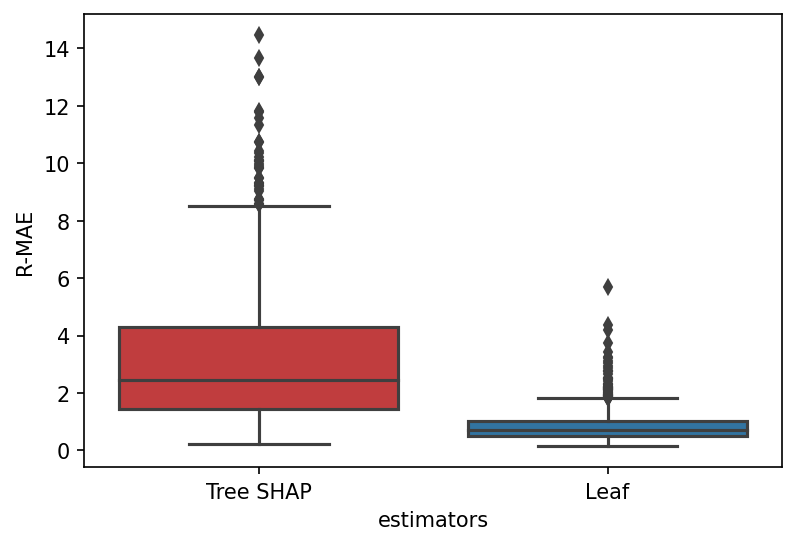}
    \caption{R-AE on 1000 new observations sampled from the synthetic model, p=5.}
    \label{fig:shap_dist_error}
\end{figure}

\begin{figure}[ht]
    \centering
    \includegraphics[width=0.7\linewidth]{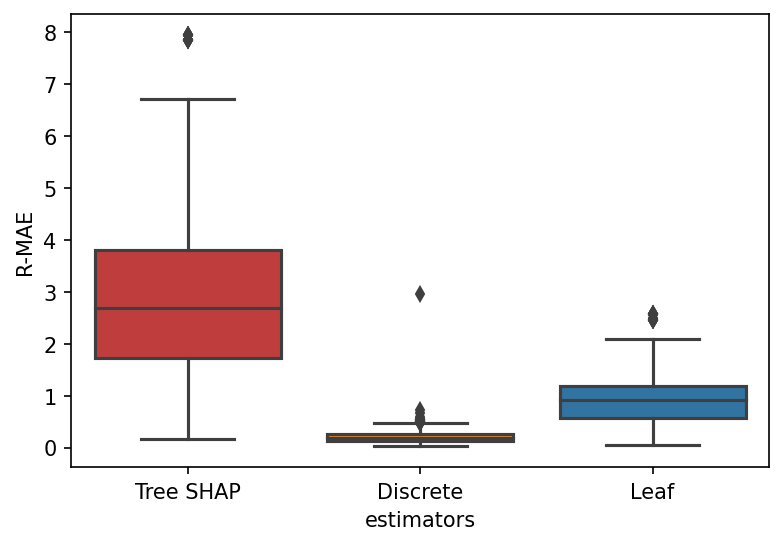}
    \caption{R-AE on 1000 new observations sampled from the synthetic model with discreted variables, p=5}
    \label{fig:discrete_comparisons}
\end{figure}

In figure \ref{fig:shap_dist_error}, we compare the SV of the Discrete unbiased estimator $\phi_{X_i}\left(\hat{f}^{(D)}\right)$, TreeSHAP $\hat{\phi_{X_i}^{SHAP}}\left(f\right)$ and Leaf estimator $\phi_{X_i}\left(\hat{f}^{(Leaf)}\right)$ with the True $\phi_{X_i}(f)$, where f was trained on the discretized version of $\mathcal{D}_n$. As demonstrated in figure \ref{fig:discrete_comparisons}, the Discrete estimator also outperforms TreeSHAP by a significant margin. 

\textbf{Experiment 2.} Here, we investigate the impact of feature dependence on the performance of the different estimators. We utilize the toy model from Experiment 1, but we vary the correlation coefficient $\rho$ from 0 to 0.99, representing increasing positive correlations among the features. As shown in Figure \ref{fig:variation}, TreeSHAP performs well when the features are independent ($\rho=0$), but it is outperformed by Leaf as the dependence between the features increases.

\begin{figure}[ht]
    \centering
    \includegraphics[width=0.7\linewidth]{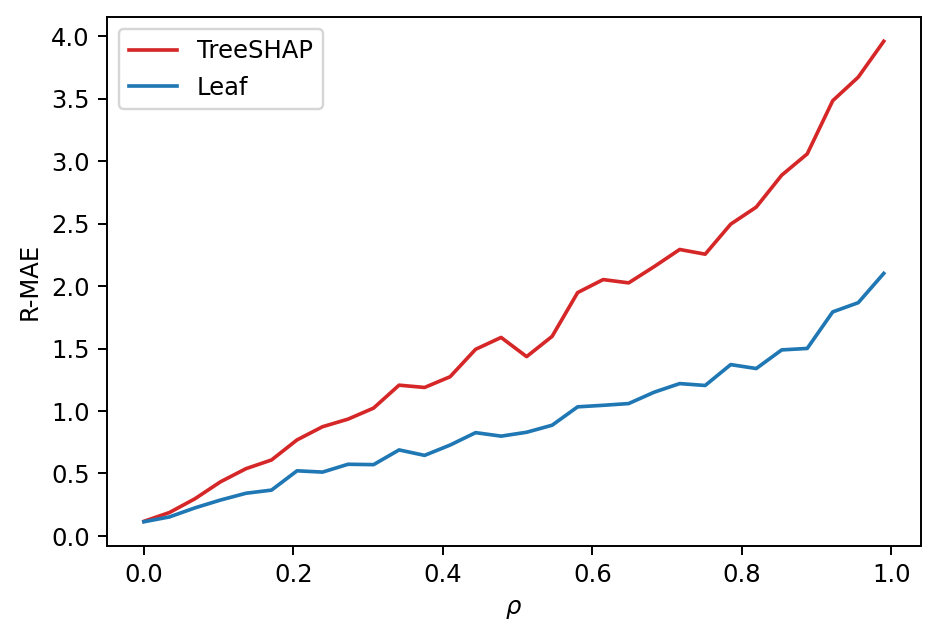}
    \caption{R-AE of the different estimators given the correlation coefficient $\rho \in [0, 0.99]$}
    \label{fig:variation}
\end{figure}
Furthermore, we conduct a runtime comparison of computing SV with Leaf and TreeSHAP on three datasets with different shapes: Boston $(n=506, p=13)$, Adults $(n=32561, p=12)$, and a toy linear model $(n=10000, p=500)$, where $n$ is the number of observations and p is the number of variables. We trained XGBoost with default parameters on these datasets and computed the SV of $1000$ observations for Adults, the toy model, and $506$ observations for Boston. As expected, as shown in Table \ref{tab:runtime}, TreeSHAP is much faster than the Leaf estimator. This difference in runtime can be partly explained by the Leaf estimator having to go through all the data for each leaf, whereas TreeSHAP uses the information stored in the trees. However, the Leaf estimator is not very affected by the dimension of the variables, as it succeeds in computing the SV when $p=500$ in a reasonable time.

\begin{table}[ht]
\caption{Run-time of TreeSHAP and Leaf estimator on Adults (A), Boston (B) and the toy (T) datasets.}
\vskip 0.1in
\begin{center}
\begin{small}
\begin{sc}
\begin{tabular}{lcccr}
\toprule
DataSETS & Leaf & Tree SHAP \\
\midrule
A (p=12)    & 1 min 4 s ± 1.73 s & 3.33 s ± 39.9 ms \\
B (p=13) & 8.82 s ± 204 ms & 129 ms  ± 6.91 ms \\
t (p=500) & 1min 5s ± 1.73 s  & 101 ms ± 4.54 ms\\
\bottomrule
\end{tabular}
\end{sc}
\end{small}
\end{center}
\label{tab:runtime}
\vskip -0.1in
\end{table}

\section{Discussion and Future works}

We have demonstrated that the current implementation of Shapley Values can lead to unreliable explanations due to biased estimators or inappropriate handling of categorical variables. To address these issues, we have proposed new estimators and provided a correct method for handling categorical variables. Our results show that even in simple models, the difference between the state-of-the-art (TreeSHAP) and proposed methods can be significant. Despite the growing interest in trustworthy AI, the impact of these inaccuracies in explanations is not well understood. One reason for this may be the difficulty in systematically and quantitatively evaluating the quality of an explanation, as it depends on the law of the data, which can be difficult to approximate. Furthermore, such analyses may be influenced by confirmation bias.

We also believe that the quality of estimates is not the only drawback of SV. In fact, we demonstrate in Proposition \ref{chap3:prop:reparam} that SV explanations are not local explanations, but remain global, even in simple linear models.
\begin{proposition}\label{chap3:prop:reparam}
Let us assume that we have $X \in \mathbb{R}^p$, $\boldsymbol{X} \in \mathcal{N}(0, I_p)$ independent Gaussian features, and a linear predictor $f$ defined as: 
\begin{align*}
 f(X) = (a_1 X_1 + a_2 X_2) \mathds{1}_{X_{5} \leq 0} + (a_3 X_3 + a_4 X_4) \mathds{1}_{X_{5} > 0}.
\end{align*}Even if we choose an observation $\boldsymbol{x}$ such that $x_5 \leq 0$  and the predictor only uses $X_1, X_2$, the SV of $\phi_{X_3}, \phi_{X_4}$ is not necessarily zero. Indeed, $\forall i \in \{3, 4\}$
{\small \begin{align*} 
  \phi_{X_i} &  = \frac{1}{p} P(\boldsymbol{X}_5 > 0) \sum_{S \subseteq [p]\setminus\{i, 5\}} \binom{p-1}{|S|}^{-1} \Big(a_i(\boldsymbol{x}_i -  E[\boldsymbol{X}_i])\Big)\\
        &  = K\Big(a_i(\boldsymbol{x}_i -  E[\boldsymbol{X}_i])\Big) \quad K\; \text{a constant}
\end{align*}}
\end{proposition}
The proof is in the Supplementary Material. Proposition \ref{chap3:prop:reparam} highlights that the SV are not a truly local measures, but rather have a global effect.  This occurs because when calculating the SV for $X_1$ or $X_2$, we also consider subsets $S$ that do not contain $X_5$. By marginalizing and changing the sign of $X_5$, we use the other linear model not used for the given observation. Such findings pose significant challenges in the interpretation of SV, and we believe that they are often overlooked due to the lack of precision and understanding of Shapley Values in practice.


\bibliography{References}


\clearpage
\appendix

\thispagestyle{empty}

\onecolumn \makesupplementtitle
\tableofcontents
\addtocontents{toc}{\protect\setcounter{tocdepth}{2}}

\newpage
\section{Proof of Proposition \ref{prop:sum_gam}: Decomposition of the original game into the sum of smaller game}

\begin{proposition}
Consider a tree-based model $f(\boldsymbol{x})=\sum_{m=1}^{M}f_{m}\mathds{1}_{L_{m}}(\boldsymbol{x})$, and let $S_m$ be the set of variables used along the path to leaf $L_m$. For any observation $\x$ and variable $X_i=x_i$, we can decompose its Shapley value $\phi_{X_i}(f^{(Leaf)})$ into the sum of $\# C(S,\boldsymbol{x})$ cooperative games defined on each leaf $m \in C(S,\boldsymbol{x})$ as follows
 \begin{equation} \label{eq:sumGames_sup}
 \phi_{X_i}\left(f^{(Leaf)}\right) = \sum_{m \in C(S,\boldsymbol{x})} \phi_{X_i}^m\left(f^{(Leaf)}\right)    
 \end{equation}
 where $\phi_{X_i}^m\left(f^{(Leaf)}\right)$ is a reweighted version of the Shapley Value of a cooperative game with players $S_m$ and value function $v(f^{(Leaf)},S) =P_X(L_m | \X_S \in L_m^S)$. 
\end{proposition}
\begin{proof}
 By definition, we have for the variable $X_i$
 \begin{align*}
     \phi_{X_i}(f^{(Leaf)}) & = \frac{1}{p}\sum_{S \subseteq [p]\setminus\{i\}} \binom{p-1}{|S|}^{-1} \bigg( f^{(Leaf)}_{S\cup i}(\boldsymbol{x}_{S\cup i}) -f^{(Leaf)}_{S}(\boldsymbol{x}_{S}) \bigg) \\
     &  = \frac{1}{p}\sum_{S \subseteq [p]\setminus\{i\}} \binom{p-1}{|S|}^{-1}  \bigg( \sum_{m \in C(S,\boldsymbol{x})} f_m \Big[ P(L_m | \boldsymbol{X}_{S\cup i} \in L_m^{S\cup i}) - P(L_m | \boldsymbol{X}_{S} \in L_m^S\Big] \bigg) \\
     & = \frac{1}{p} \sum_{m \in C(S,\boldsymbol{x})} \sum_{S' \subseteq S_m\setminus \{i\}} \textcolor{red}{\bigg[} \binom{p-1}{|S'|}^{-1} f_m \Big[ P(L_m | \boldsymbol{X}_{S'\cup i} \in L_m^{S'\cup i})   - P(L_m | \boldsymbol{X}_{S'} \in L_m^{S'}) \Big] \\ 
     & + \sum_{Z\neq \emptyset, Z \subseteq\,  \overline{S_m\cup i} } \binom{p-1}{|Z| + |S'|}^{-1} f_m  \Big[P(L_m | \boldsymbol{X}_{S'\cup Z\cup i} \in L_m^{S'\cup Z\cup i}) - P(L_m |\boldsymbol{X}_{S'\cup Z} \in  L_m^{S_m\cup Z}) \Big] \textcolor{red}{\bigg]}\\
 \end{align*}
If $Z \subseteq \Bar{S}_m$ and  $S \subseteq S_m$, we have
  \begin{equation}
     P_X(L_m | \boldsymbol{X}_{Z \cup S} \in L_m^{Z\cup S}) = P_X(L_m | \boldsymbol{X}_{S} \in L_m^S). \label{eq:simplification}
 \end{equation}
 Therefore, $\phi_{X_i}$ can be rewrite as:
 \begin{align*}
     & \phi_{X_i}(f^{(Leaf)}) \\
     &  =   \frac{1}{p} \sum_{m \in C(S,\boldsymbol{x})} \sum_{S' \subseteq S_m\setminus \{i\}} \bigg[ \binom{p-1}{|S'|}^{-1} + \sum_{Z\neq \emptyset, Z \subseteq\,  \overline{S_m\cup i} } \binom{p-1}{|Z| + |S'|}^{-1}\bigg] \\
     & \times f_m  \Big[ P(L_m | \boldsymbol{X}_{S'\cup i} \in L_m^{S'\cup i}) - p(L_m | \boldsymbol{X}_{S'} \in L_m^{S'}) \Big] \\
     &  \triangleq \sum_{m \in C(S,\boldsymbol{x})} \phi_{X_i}^m(\boldsymbol{x})
 \end{align*}
\end{proof}  
\section{Proof of SV invariance for transformed continuous variables}
\begin{proposition}\label{prop:reparam_sup}
Let $f$ and $\Tilde{f} = f \circ \boldsymbol{\varphi}^{(-1)} $ its reparametrization, then we have $\forall i \in \left\llbracket 1,p\right\rrbracket$, and $\boldsymbol{U} = \boldsymbol{\varphi}(\X)$: 
\begin{align*}
     \phi_{X_i}(f) = \phi_{U_i}(\tilde{f}). \label{eq:equivariance} 
\end{align*}
\end{proposition}
\begin{proof}
It is a direct application of the change of variables formula. If $g(\x)$ is the joint density of $X_1,\dots,X_p$, the transformed variable  $\boldsymbol{U} = (\varphi_1(X_1),\dots, \varphi_p(X_p))$ has density $\Tilde{g}(\boldsymbol{u}) = g\circ\boldsymbol{\varphi}^{(-1)}(\boldsymbol{u}) \times \prod_i \vert J(\varphi_i^{(-1)})(u_i)\vert $. We have \[\Tilde{g}(u_{\bar{S}} \vert u_{S}) = \frac{\Tilde{g}(u_{\bar{S}} ,u_{S})}{\Tilde{g}_S(u_{S})}= g \left(\varphi_{\bar{S}}^{(-1)}(u_{\bar{S}})\vert \varphi_S^{(-1)}(u_S)\right) \times \prod_{i\in \bar{S}}  \vert J(\varphi_i^{(-1)})(u_i)\vert. \]

The computation of the reduced predictor is straightforward
\begin{align*}
E\left[f(\boldsymbol{X})\vert \XS = \boldsymbol{x}_{S}\right] & =\int f(\boldsymbol{x}_{S},\boldsymbol{x}_{\bar{S}})g(\boldsymbol{x}_{\bar{S}}\vert\boldsymbol{x}_{S})d\boldsymbol{x}_{\bar{S}}\\
 & =\int f\Big(\varphi_{S}^{(-1)}\circ\varphi_{S}(\boldsymbol{x}_{S}),\varphi_{\bar{S}}^{(-1)}\circ\varphi_{\bar{S}}(\boldsymbol{x}_{\bar{S}})\Big)\:g(\boldsymbol{x}_{\bar{S}}\vert\boldsymbol{x}_{S})d\boldsymbol{x}_{\bar{S}}\\
 & =\int\tilde{f}(\boldsymbol{u}_{S},\boldsymbol{u}_{\bar{S}})g\left(\varphi_{\bar{S}}^{(-1)}(\boldsymbol{u}_{\bar{S}})\vert\varphi_{S}^{(-1)}(\boldsymbol{u}_{S})\right)\prod_{i\in\bar{S}}\left|J\varphi^{(-1)}(u_{i})\right| d\boldsymbol{u}_{\bar{S}}\\
 & =E\left[\tilde{f}(\boldsymbol{U}_{S},\boldsymbol{U}_{\bar{S}})\vert\boldsymbol{U}_{S}=\boldsymbol{u}_{S}\right].
\end{align*}

The equality of Shapley Values directly results from the equality of reduced predictors.

\end{proof}

\section{Proof of SV invariance for encoded categorical variables}
\begin{proposition}\label{prop:SV_CoalQualitative_sup}
 Given a predictor  $f:\mathbb{R}\times\left\{ 1,\dots,K\right\} \longrightarrow\mathbb{R}$ and its reparametrization $\tilde{f}$ using Dummy Encoding  $\tilde{f}:\mathbb{R}\times\left\{ 0,1\right\}^{K-1} \longrightarrow\mathbb{R}$ such that $f(X,Z)\triangleq\tilde{f}(X,Z_{1},\dots,Z_{K-1})$, we have
 \begin{equation}
\left\{ \begin{array}{lll}
 \phi_{Z_{1:K-1}}(\tilde{f}) & = & \phi_{Z}(f)\\
 \phi_{X}(\tilde{f}; Z_{1:K-1}) & = & \phi_{X}(f),
\end{array}\right.\label{eq:Shapley_CoalitionDummies}
\end{equation}
\end{proposition}
We recall the expression of the SV of the  two variables $X\in\mathbb{R}$
and $Z\in\left\{ 1,\dots,K\right\}$. The roles of variables $X, Z$ are symmetric, and the categorical or quantitative nature of the variable does not have any impact on the computation of the SV, as demonstrated below. Let's consider an observation $\x = (x, z)$, then
\begin{eqnarray}
\left\{ \begin{array}{ll}
\phi_{X}(f)= & \frac{1}{2}\left(E\left[f(X,Y)\vert X=x\right]-E\left[f\left(X,Z\right)\right]\right)+\frac{1}{2}\left(f\left(x,z\right)-E\left[f\left(X,Z\right)\vert Z=z\right]\right)\\
\phi_{Z}(f)= & \frac{1}{2}\left(E\left[f(X,Y)\vert Z=z\right]-E\left[f\left(X,Z\right)\right]\right)+\frac{1}{2}\left(f\left(x,z\right)-E\left[f\left(X,Z\right)\vert X=x\right]\right)
\end{array}\right.
\label{eq:ShapleyValue_XY_Indiv}
\end{eqnarray}

\begin{proof}
Let's consider the Dummy Encoding (DE) $\varphi:z\mapsto(z_{1},\dots,z_{K-1})$ without loss of generality, then the observation $\x = (x, z)$ is reparametrized as $\Tilde{\x} = (x, z_{1:K-1})$, and by construction of $\varphi$, $\exists !z\in \{1,\dots,K\}$ such
that $\varphi(z)=z_{1:K-1}$. By taking the coalition of $Z_{1:K-1}$ or considering them as a single variable, we have
\begin{multline}
    \phi_{Z_{1:K-1}}(\tilde{f})=\frac{1}{2}\left( E_{\tilde{P}}\left[\tilde{f}(X,Z_{1:K-1})\vert Z_{1:K-1}=z_{1:K-1}\right]-E_{\tilde{P}}\left[\tilde{f}(X,Z_{1:K-1})\vert\emptyset\right] \right) \\
    +\frac{1}{2} \left( E_{\tilde{P}}\left[\tilde{f}(X,Z_{1:K-1})\vert X=x,Z_{1:K-1}=y_{1:K-1}\right]-E_{\tilde{P}}\left[\tilde{f}(X,Z_{1:K-1})\vert X=x\right]\right).
\end{multline}

Using the fact that 
\begin{align*}
 E_{\tilde{P}}\left[\tilde{f}(X,Z_{1:K-1})\vert Z_{1:K-1}=z_{1:K-1}\right] & =\int\tilde{f}(x,z_{1:K-1})dP(x\vert z_{1:K-1})\\
 & =\int\tilde{f}(x,z_{1:K-1})\frac{dP(x,z_{1:K-1})}{P(z_{1:K-1})}\\
 & =\int\tilde{f}(x,\varphi(z))\frac{dP(x,\varphi(z))}{P(\varphi(z))}\\
 & =\int f(x,z)\frac{dP(x,z)}{P(z)} \\
 & = E\left[f(X,Z)\vert Z=z\right]
\end{align*} 
we have 
\begin{align*}
 E_{\tilde{P}}\left[\tilde{f}(X,Z_{1:K-1})\vert Z_{1:K-1}=z_{1:K-1}\right]-E_{\tilde{P}}\left[\tilde{f}(X,Z_{1:K-1})\vert\emptyset\right]
 & =E_{P}\left[\tilde{f}\left(X,\varphi\left(Z\right)\right)\vert Z=z\right]-E_{P}\left[\tilde{f}\left(X,\varphi\left(Z\right)\right)\right]\\
 & =E_{P}\left[f\left(X,Z\right)\vert Z=z\right]-E_{P}\left[f\left(X,Z\right)\right]
\end{align*}
We also have
\begin{align*}
& E_{\tilde{P}}\left[\tilde{f}(X,Z_{1:K-1})\vert X=x,Z_{1:K-1}=z_{1:K-1}\right]-E_{\tilde{P}}\left[\tilde{f}(X,Z_{1:K-1})\vert X=x\right] \\
& =\tilde{f}(x,z_{1:K-1})-E_{P}\left[\tilde{f}(X,\varphi\left(Z\right))\vert X=x\right]\\
 & =\tilde{f}(x,\varphi\left(z\right))-E_{P}\left[\tilde{f}(X,\varphi\left(Z\right))\vert X=x\right]\\
 & =f(x,z)-E_{P}\left[f(X,Z)\vert X=x\right]
\end{align*}
Consequently, we have
\begin{align*}
\phi_{Z_{1:K-1}}(\tilde{f}) & =\frac{1}{2}\left(E_{P}\left[f\left(X,Z\right)\vert Z=z\right]-E_{P}\left[f\left(X,Z\right)\right]\right)\\
 & \:+\frac{1}{2}\left(f(x,z)-E_{P}\left[f(X,z)\vert X=x\right]\right)
\end{align*}
We can recognize that we have exactly $\phi_{Z_{1:K-1}}(\tilde{f})=\phi_{Z}(f)$. We can derive similarly that $\phi_{X}(\tilde{f};Z_{1:K-1})=\phi_{X}(f)$. 
\end{proof}

\begin{proposition} If $X \sim \mathcal{N}(\mu, \Sigma)$, then $X_{\bar{S}} | X_{S} = x_{S}$ is also multivariate gaussian with mean $\mu_{\bar{S} | S}$ and covariance matrix $\Sigma_{\bar{S} | S}$ equal:
\begin{equation*}
    \mu_{\bar{S} | S} = \mu_{\bar{S}} + \Sigma_{\bar{S}, S}\Sigma_{S, S}^{-1}(x_S - \mu_S) \; \text{ and  } \Sigma_{\bar{S} | S} = \Sigma_{\bar{S} \bar{S}} - \Sigma_{\bar{S} S} \Sigma_{S S}^{-1} \Sigma_{S, \bar{S}} 
\end{equation*}
\end{proposition}

\section{Proof of the limitation of SV as local explanation} \label{sec:local_sv}
\begin{proposition}\label{prop:a_reparam}
Let us assume that we have $X \in \mathbb{R}^p$, $\boldsymbol{X} \in \mathcal{N}(0, I_8)$ and a linear predictor $f$ defined as: 
\begin{align}
 f(X) = (a_1 X_1 + a_2 X_2) \mathds{1}_{X_{5} \leq 0} + (a_3 X_3 + a_4 X_4) \mathds{1}_{X_{5} > 0}.
\end{align}
Even if we choose an observation $\boldsymbol{x}$ such that $x_5 \leq 0$  and the predictor only uses $X_1, X_2$, the SV of $\phi_3, \phi_4$ is not necessarily zero.
\end{proposition}

\begin{proof}
\begin{align}
    \phi_3 & = \small \frac{1}{p}\sum_{S \subseteq [p]\setminus\{3\}} \binom{p-1}{|S|}^{-1} \bigg( f_{S\cup 3}(\boldsymbol{x}_{S\cup 3}) -f_{S}(\boldsymbol{x}_{S}) \bigg) \\
    & = \small \frac{1}{p}\sum_{S \subseteq [p]\setminus\{3, 5\}} \binom{p-1}{|S|}^{-1} \bigg( f_{S\cup 3}(\boldsymbol{x}_{S\cup 3}) -f_{S}(\boldsymbol{x}_{S}) \bigg) + 
    \frac{1}{p}\sum_{S \subseteq [p]\setminus\{3, 5\}} \binom{p-1}{|S|+1}^{-1} \bigg( f_{S\cup \{3, 5\}}(\boldsymbol{x}_{S\cup \{3, 5\}}) -f_{S \cup 5}(\boldsymbol{x}_{S \cup 5}) \bigg)
\end{align} 
The second term is zero. Indeed, $\forall S \subseteq [p]\setminus\{3, 5\}$
\begin{align*}
    f_{S\cup \{3, 5\}}(\boldsymbol{x}_{S\cup \{3, 5\}}) -f_{S \cup 5}(\boldsymbol{x}_{S \cup 5}) = 0
\end{align*}
Because, if we condition on the event $\{X_5 = \boldsymbol{x}_5\}$ with $x_5 \leq 0$ 
\begin{align*}
    f_{S\cup \{3, 5\}}(\boldsymbol{x}_{S\cup \{3, 5\}}) & = E\left[(a_1 X_1 + a_2 X_2) \mathds{1}_{X_{5} \leq 0} + (a_3 X_3 + a_4 X_4) \mathds{1}_{X_{5} > 0} \;|\; X_{S\cup \{3, 5\}}  = \boldsymbol{x}_{S\cup \{3, 5\}}\right] \\
    & = E\left[(a_1 X_1 + a_2 X_2) \mathds{1}_{X_{5} \leq 0} \;|\; X_{S\cup \{3, 5\}} = \boldsymbol{x}_{S\cup \{3, 5\}}\right] && \text{because $x_5 \leq 0$}\\ 
    & = E\left[(a_1 X_1 + a_2 X_2) \; |\; X_{S\cup 5} = \boldsymbol{x}_{S\cup 5}\right] && \perp \!\!\! \perp\text{of $X_3$} \\
    & = f_{S \cup 5}(\boldsymbol{x}_{S \cup 5}) 
\end{align*}
\end{proof}

The first term of 3.3 is the classic marginal contribution of SV in linear model. $\forall S \subseteq [p]\setminus\{3, 5\}$

\begin{align*} 
f_{S\cup 3}(\boldsymbol{x}_{S\cup 3}) 
& = E\left[a_1 X_1 + a_2 X_2 \;| X_{S\cup 3}=\boldsymbol{x}_{S\cup 3}\right] P(X_5 \leq 0 \;| X_{S\cup 3}=\boldsymbol{x}_{S\cup3}) \\
& + E\left[a_3 X_3 + a_4 X_4 \;| X_{S\cup 3}=\boldsymbol{x}_{S\cup 3}\right] P(X_5 > 0 | X_{S\cup 3} = \boldsymbol{x}_{S\cup3}) \\
& = E\left[a_1 X_1 + a_2 X_2 | X_{S}=\boldsymbol{x}_{S} \right]P(X_5 \leq 0) +  (E\left[a_2 X_2 | X_{S}=\boldsymbol{x}_{S} \right]+ a_3 \boldsymbol{x}_3)P(X_5 > 0) \\
& =  f_{S}(\boldsymbol{x}_{S}) + P(X_5 > 0)\Big(a_3(\boldsymbol{x}_3 -  E[X_3])\Big)
\end{align*}

Therefore, 
\begin{align*}
  \phi_3 & = \frac{1}{p}\sum_{S \subseteq [p]\setminus\{3, 5\}} \binom{p-1}{|S|}^{-1} P(X_5 > 0)\Big(a_3(\boldsymbol{x}_3 -  E[X_3])\Big)\\
        & = K\Big(a_3(\boldsymbol{x}_3 -  E[X_3])\Big) && K\; \text{is a constant}
\end{align*}

The computation of $\phi_4$ is obtained by symmetry.

\section{Relation between the Algorithm 1 (TreeSHAP with path-dependent) and $\hat{f}^{SHAP}$}
In section 3.1, we have said that the recursive algorithm 1 introduced in \cite{lundberg2020local2global} and shows in figure 2 assumes that the probabilities can be factored with the decision tree as: 
\begin{align} 
 P_X^{SHAP} \left( \prod_{k=1}^{d_m} I_{N_k} | \XS=\xs\right) =  \delta_S(N_1) \times \prod_{i=2\vert N_i \notin S}^{d_m }  P\left(X_{N_{i}} \in I_{N_{i}} \Big| \prod_{k=2\vert N_k \notin S}^{i} X_{N_{k-1}} \in I_{N_{k-1}}\right) \label{eq:ProbaSHAP_sup}
\end{align}
with $\delta_S(N_1)=P(X_{N_{1}} \in I_{N_{1}})$ if $N_1\notin S$, and 1 otherwise. 
\begin{figure}[ht] \label{fig:tree_shap}
    \begin{minipage}[b]{0.3\linewidth}
        \centering
        \includegraphics[width=10cm]{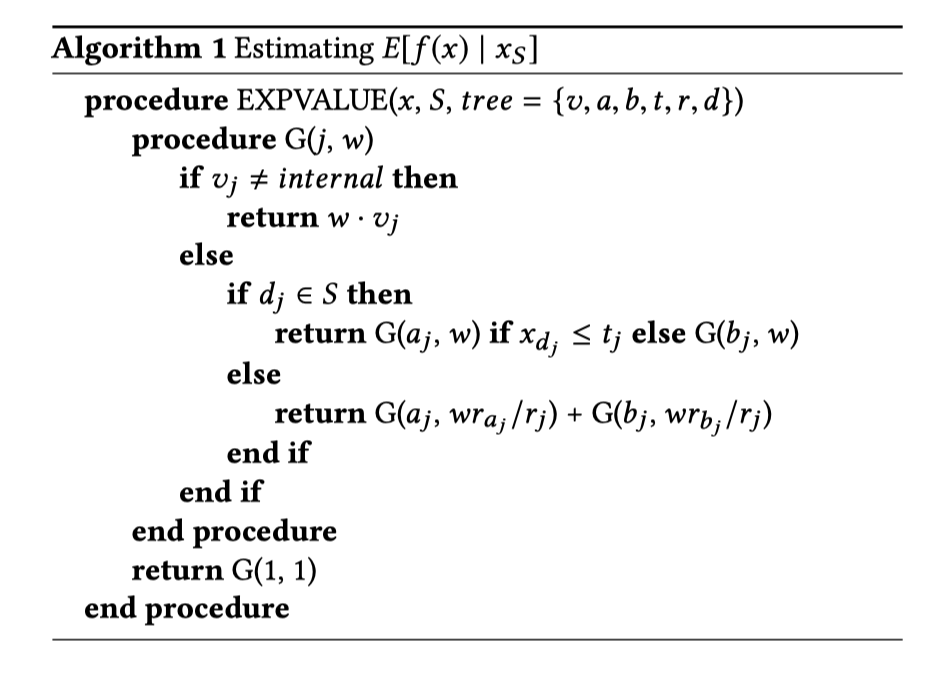}
    \end{minipage}
    \hspace{0.5cm}
    \begin{minipage}[b]{0.3\linewidth}
        \centering
        \includegraphics[width=11.3cm]{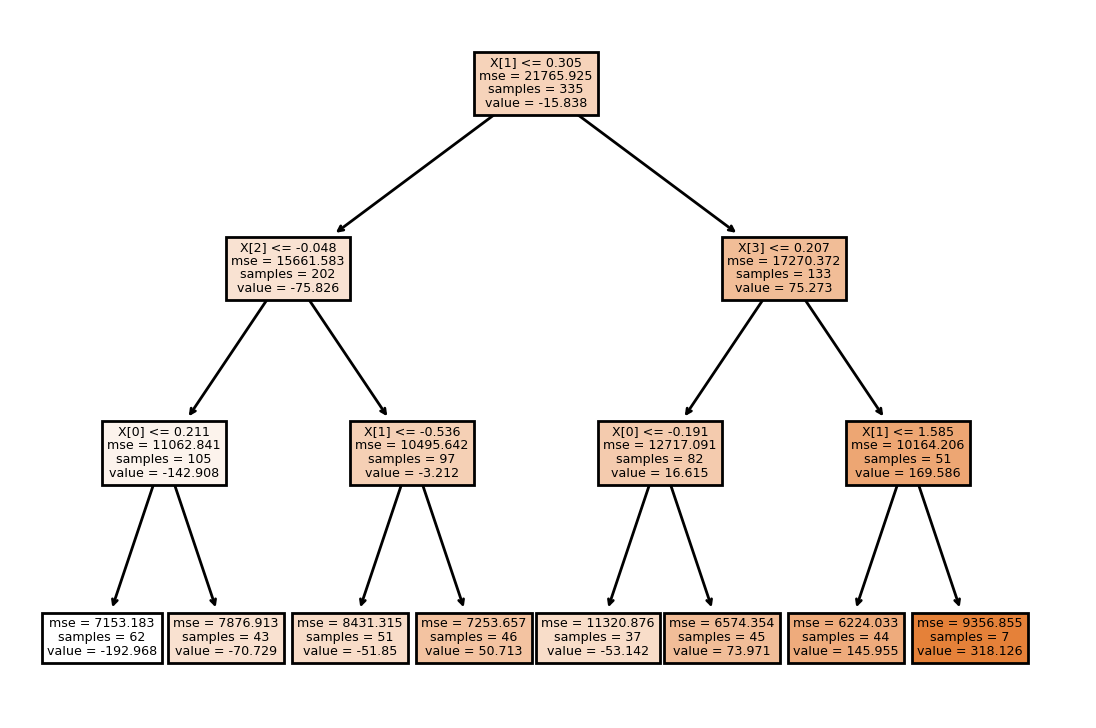}
    \end{minipage}
    \caption{Left figure: Algorithm 1 in \cite{lundberg2020local2global} (Tree SHAP). Right figure: A sample decision tree used to illustrate the link between $\hat{f}^{(SHAP)}$ }
\end{figure}

To show the link between between $\hat{f}^{SHAP}$ and the Algorithm 1, we choose an observation $x = (2, 3, 0.5, -1)$ and compute $E[\hat{f}^{SHAP}(X) | \; X_0=2, X_2=0.5]$ where $f$ is the tree in the left of figure  \ref{fig:tree_shap}. $x$ is comptatible with Leaf $6, 7, 11, 13, 14$, we denote $f_6, f_7, f_{11}, f_{13}, f_{14}$ the value of each leaf respectively. 

The output of the algorithm is on step 4, and its corresponds to: 

\begin{align*}
    \hat{f}^{(SHAP)}(x) & = P(X_1 \leq 0.305) P(X_2 > -0.048 | X_1 \leq 0.305) * P(X_1 \leq -0.536 | X_2 > -0.048) f_6\\
    & + P(X_1 \leq 0.305) P(X_2 > -0.048 | X_1 \leq 0.305) * P(X_1 > -0.536 | X_2 > -0.048) f_7\\
    & + P(X_1 > 0.305) P(X_3 \leq 0.207 | X_1 > 0.305) * P(X_0 > -0.191 | X_3 \leq 0.207) f_{11} \\
    & + P(X_1 > 0.305) P(X_3 > 0.207 | X_1 > 0.305) * P(X_1 \leq 1.585 | X_3 > 0.207) f_{13} \\
    & + P(X_1 > 0.305) P(X_3 > 0.207 | X_1 > 0.305) * P(X_1 > 1.585 | X_3 > 0.207) f_{14} \\
    & = (202/335) * 1 * (51/97) * (-51.85) + (202/335) * 1 * (46/97) * (50.716) \\
    & +  (133/335) * (82/133) * 1 * (73.971) + (133/335) * (51/133) * (44/51) * (145.955) \\
    & + (133/335) * (51/133) * (7/51) * (318.125) \\
    & = 41.98
\end{align*}

\begin{table}[ht]
\begin{tabular}{@{}|l|l|@{}}
\toprule
Step & Calculus                                     \\ \midrule
0    & G(0, 1)                                      \\ \cmidrule(r){1-1}
1    & G(1, 202/335) + G(8, 133/335)                \\ \cmidrule(r){1-1}
2    & G(5, 202/335) + G(9, 88/335) + G(12, 51/335) \\ \cmidrule(r){1-1}
3 & G(6,(202/335)*(51/97)) + G(7,(202/335)*(46/97)) + G(11,82/335) + G(13,44/335) \\
& + G(14,7/335)                \\ \cmidrule(r){1-1}
4 & -(202/335)*(51/97)*51,85 + (202/335)*(46/97)*50,713 + (82/335)*73,971 \\
& + (44/335)*145,955 + (7/335)*318,126 \\ \cmidrule(r){1-1}
5 & = 41.98 \\ \bottomrule
\end{tabular}
\end{table}

\newpage
\section{Additional experiments}

\subsection{Impact of quantile discretization}
The table below shows the impact of discretization on the performance of a Random Forest on UCI datasets.
\begin{table}[ht]
\centering
\begin{tabular}{@{}|l|lllll|@{}}
\toprule
\textbf{Dataset} &
  \multicolumn{1}{l|}{\textbf{Breiman's RF}} &
  \multicolumn{1}{l|}{\textbf{q=2}} &
  \multicolumn{1}{l|}{\textbf{q=5}} &
  \multicolumn{1}{l|}{\textbf{q=10}} &
  \textbf{q=20} \\ \midrule
Authentification & 0.0002 & 0.08 & 0.002 & 0.0005 & 0.0004 \\ \cmidrule(r){1-1}
Diabetes         & 0.17   & 0.23 & 0.18  & 0.18   & 0.18   \\ \cmidrule(r){1-1}
Haberman         & 0.32   & 0.35 & 0.30  & 0.32   & 0.30   \\ \cmidrule(r){1-1}
Heart Statlog    & 0.10   & 0.10 & 0.10  & 0.10   & 0.10   \\ \cmidrule(r){1-1}
Hepastitis       & 0.13   & 0.15 & 0.14  & 0.14   & 0.13   \\ \cmidrule(r){1-1}
Ionosphere       & 0.02   & 0.07 & 0.03  & 0.02   & 0.02   \\ \cmidrule(r){1-1}
Liver Disorders  & 0.23   & 0.32 & 0.27  & 0.25   & 0.24   \\ \cmidrule(r){1-1}
Sonar            & 0.07   & 0.09 & 0.07  & 0.07   & 0.07   \\ \cmidrule(r){1-1}
Spambase         & 0.01   & 0.14 & 0.03  & 0.02   & 0.01   \\ \cmidrule(r){1-1}
Titanic          & 0.13   & 0.15 & 0.14  & 0.14   & 0.13   \\ \cmidrule(r){1-1}
Wilt             & 0.007  & 0.15 & 0.03  & 0.02   & 0.02   \\ \bottomrule
\end{tabular}
\caption{Accuracy, measured by 1-AUC on UCI datasets, for two algorithms: Breiman’s random forests and random forests with splits limited to q-quantiles, for $q \in \{2,5,10,20\}$. Table 5 in \cite{benard2021sirus}}
\end{table}

\subsection{The differences between Coalition and sum on Census Data}
We use UCI Adult Census Dataset \cite{UCIADULT:2019}. We keep only 4 highly-predictive categorical variables: Marital Status, Workclass, Race, Education and use a Random Forest which has a test accuracy of $86\%$. We compare the Global SV by taking the coalition or sum of the modalities. Global SV are defined as $I_j = \sum_{i=0}^{N} \vert \phi_{X_j}^{(i)} \vert$.
\begin{figure}[ht]
    \begin{minipage}[b]{0.45\linewidth}
        \centering
        \includegraphics[width=6cm]{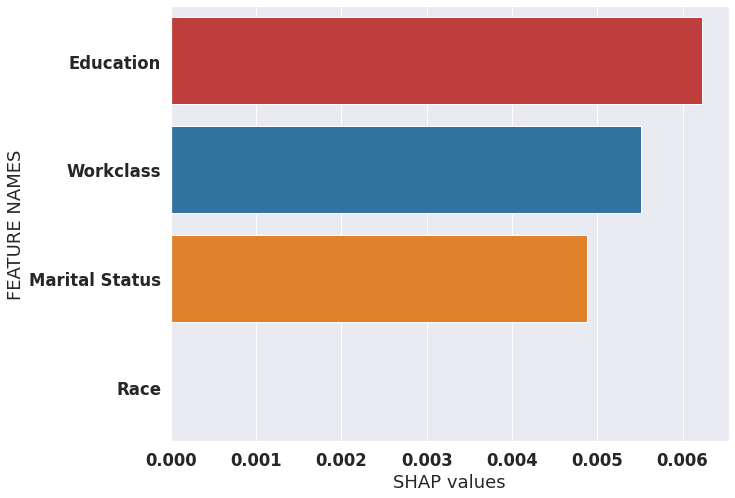}
    \end{minipage}
    \hspace{0.5cm}
    \begin{minipage}[b]{0.45\linewidth}
        \centering
        \includegraphics[width=6cm]{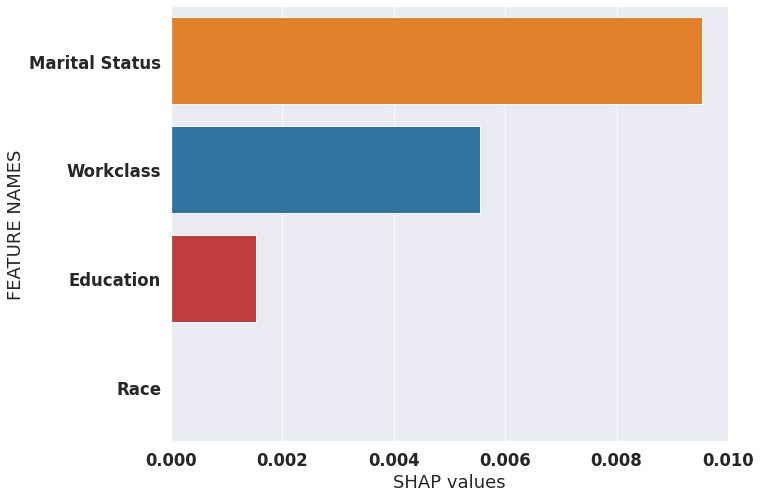}
    \end{minipage}
    \caption{Difference between the global absolute value of SV: sum (left) vs coalition (right) of dummies of individual with modalities: Married, local gov, others, 1rst-4th.}
    \label{fig:mean_shap_value_examples}
\end{figure}

In figure \ref{fig:mean_shap_value_examples}, we see differences between the global SV with coalition and sum with N=5000. The ranking of the variables changes, e.g. Education goes from important with sum to not important with the coalition. We also compute the proportion of order inversion over N=5000 observations choose randomly. The ranking of variables is changed in $10\%$ of the cases. Note that this difference may increase or diminish depending on the data.

\section{EXPERIMENTAL SETTINGS} \label{sec:exp}
\setcounter{section}{1}
\setcounter{subsection}{0}
All our experiments are reproducible and can be found on the github repository \emph{Active Coalition of Variables}, \url{https://github.com/salimamoukou/acv00}

\subsection{Toy model of Section 2.3} \label{appendix:toy_model}
Recall that the model is a linear predictor with categorical variables define as $f(X, Y) = B_Y X$ with $X|Y=y \sim \mathcal{N}(\mu_y, \Sigma_y)$ and $\mathbb{P}(Y=y)= \pi_y$, $Y \in \{a, b, c\}$.\\
For the experiments in Figure 1 and 2, we set $\pi_y = \frac{1}{3}, \mu_y = 0,  \forall \; y \in \{a, b, c\}$. We use a random matrices generated from a Wishart distribution. The covariance matrices  are:\\
$\Sigma_a=$ $\begin{bmatrix}
0.41871254 & -0.790061361 & 0.46956991\\
-0.79006136 &  1.90865098 & -0.82571655\\
0.46956991 &  -0.82571655 &   0.95835472
\end{bmatrix}$
,  $\Sigma_b=$$\begin{bmatrix}
0.55326081 & 0.11811951& -0.70677924]\\
0.11811951&  2.73312979& -2.94400196\\
-0.70677924& -2.94400196&  4.22105088
\end{bmatrix}$, $\Sigma_c=$
$\begin{bmatrix}
9.2859966&  1.12872646&  2.4224434\\
1.12872646&  0.92891237& -0.14373393\\
2.4224434& -0.14373393&  1.81601676
\end{bmatrix}$  for $y \in \{a, b, c\}$ respectively. \newline
The coefficients are $B_a =[1, 3, 5], B_b = [-5, -10, -8], B_c = [6, 1, 0]$ and the selected observation  in figure 1 values is $x = [ 0.35, -1.61, -0.11, 1.,  0. , 0.]$

\subsection{Toy model of Section 4}
The data $\mathcal{D}=\left(x_{i},z_{i}\right)_{1\leq i\leq n}$ are generated from a linear regression $Z = B^t X$ with  $n = 10 000$, $\boldsymbol{X} \in \mathbb{R}^{p}$, $\boldsymbol{X} \in \mathcal{N}(0, \Sigma)$, $\Sigma = \rho  J_p + (\rho - 1) I_p$ with $p=5, \rho=0.7$, $I_p$ is the identity matrix, $J_p$ is all-ones matrix and a linear predictor  $Z = B^t \boldsymbol{X}$. $B = [6.49, -2.44, -2.11, -4.29, 3.46]$ for the continuous case and d=3, $B = [6.49, -2.44, 0]$ for the discrete case.

We used the decision tree of sklearn trained on $\mathcal{D}$ with the defaults parameters. The Mean Squared Error (MSE) are MSE $= 4.39$ for the continuous case and MSE $= 2.88$ for the discrete case.

\end{document}